\documentclass[a4paper,10pt]{article}

\usepackage[utf8]{inputenc}
\usepackage[english]{babel}
\usepackage{euscript}

\usepackage{graphicx}
\usepackage{amsmath}
\usepackage{amssymb}
\usepackage{ntheorem}

\usepackage{pictexwd,dcpic}
\usepackage[all]{xy}

\usepackage{authblk}

\def\quotient#1#2{%
    \raise1ex\hbox{$#1$}\Big/\lower1ex\hbox{$#2$}%
}

\newtheorem*{defin}{Definition}
\newtheorem{lemma}{Lemma}

\newcommand{\stirling}[2]{\genfrac{\{}{\}}{0pt}{}{#1}{#2}}
\newcommand{\probP}[0]{\text{P}} 
\newcommand{\falling}[2]{({#1})_{{#2}}}
\newcommand{\nchoosek}[2]{\genfrac{(}{)}{0pt}{}{#1}{#2}}

\newenvironment{fproof}[1][Derivation of raw integer moments]{\begin{trivlist}
\item[\hskip \labelsep {\bfseries #1}]}{\end{trivlist}}
\newenvironment{proof}[1][Proof]{\begin{trivlist}
\item[\hskip \labelsep {\bfseries #1}]}{\end{trivlist}}
\newenvironment{derivation}[1][Derivation]{\begin{trivlist}
\item[\hskip \labelsep {\bfseries #1}]}{\end{trivlist}}

\title{What is the distribution of the number of unique original items in a bootstrap sample?}

\author[1]{Alex F. Mendelson\thanks{alexander.mendelson.11@ucl.ac.uk}}
\author[1]{Maria~A.~Zuluaga}
\author[2,3]{Brian~F.~Hutton}
\author[1,4]{S\'{e}bastien~Ourselin}

\affil[1]{Translational Imaging Group, Centre for Medical Image Computing, University College London, NW1 2HE London, UK}
\affil[2]{Institute of Nuclear Medicine, University College London, NW1 2BU London, UK}
\affil[3]{Centre for Medical Radiation Physics, University of Wollongong, 2522 NSW, Australia}
\affil[4]{Dementia Research Centre, University College London, WC1N 3BG London, UK }

\begin{document}

\maketitle

\begin{abstract}

Sampling with replacement occurs in many settings in machine learning, notably in the bagging ensemble technique and the .632+ validation scheme. The number of unique original items in a bootstrap sample can have an important role in the behaviour of prediction models learned on it. Indeed, there are uncontrived examples where duplicate items have no effect. The purpose of this report is to present the distribution of the number of unique original items in a bootstrap sample clearly and concisely, with a view to enabling other machine learning researchers to understand and control this quantity in existing and future resampling techniques. We describe the key characteristics of this distribution along with the generalisation for the case where items come from distinct categories, as in classification. In both cases we discuss the normal limit, and conduct an empirical investigation to derive a heuristic for when a normal approximation is permissible.


\end{abstract}



\section{Introduction}

Bootstrap resampling, or sampling with replacement from the given data, is used to mimic the sampling process that produced the original sample in the first place. By randomly drawing items from the original dataset with equal probability, one is effectively drawing a fresh sample from the ``empirical distribution'', a multinomial with equal probability assigned to each item in the original sample. This can be seen as a best guess at the true population distribution. It has been over 30 years since \cite{efron1979bootstrap} introduced the bootstrap to statistics, and now there exist many different applications and related resampling techniques are found in statistics and computer science. Of particular interest to researchers in pattern recognition and machine learning are the bagging ensemble method of \cite{Bagging1996} and the .632+ validation scheme of \cite{632p} for supervised algorithms.

In these settings, and others involving the construction of prediction rules on bootstrap samples, an important quantity of interest is the number of unique items from the original sample \cite{Rao1997seqres}. In many cases, this can be seen as a limit on the amount of information carried down from the original sample. For learning algorithms trained on bootstrap samples, the reduction in the number of unique items expected can be viewed as an effective reduction in training sample size, particularly in high dimensional problems where the prediction model is unstable with respect to changes in the training data. 

While generally it is a desirable property of learning algorithms that they are flexible and responsive to new information, there also exist specific cases where the unique items in a training sample are the sole determinant of the prediction rule. In (hard-margin) support vector machine classification where the number of dimensions exceeds the number of items observed, it is always possible to divide the two classes in the training data. The prediction rule is determined only by the points that lie on the maximum margin separating hyperplane, whose duplication has no effect on its location. Another simple example is found in nearest-neighbour classification and regression in continuous feature spaces where the probability of obtaining any truly identical training items is zero. Without competitions between multiple prediction values due to items at exactly the same location, the resulting prediction rule is unaffected by duplications.

As the number of unique items in a bootstrap sample is an important determinant of the behaviours of prediction rules learned on it, the distribution of this quantity should be of interest to researchers working on their development and validation. While related distributions have long been studied in a purer mathematical context \cite{Urns}, and this distribution has been identified before in this setting \cite{Outlier,Abadie}, nowhere were we able to find a concise and accessible summary of the relevant information for the benefit of researchers in machine learning. Our aim here is to fill this gap by presenting this distribution along with its key properties, and to make it easier for others who to understand or modify resampling techniques in a machine learning context.

In section \ref{background}, we discuss the relevance of the number of unique items to the bagging ensemble method and the .632+ validation scheme. In section \ref{present_dist} we give a closed form of this distribution together with its notable properties. We give an empirically derived rule to decide whether normal approximation is permissible, and describe how we produced it. In section \ref{multi}, we consider the case where the items in the original sample belong to one of several categories, as is the case in classification problems. Here, where the outcome is the vector of the number of unique original items from each category, we show that the limit distribution is multivariate normal, and consider limits to produce a heuristic for the normal approximation of the number of unique items from a single category.

\section{Background}

\label{background}

\subsection{Bagging}

Bootstrap aggregation, or bagging, uses the perturbed samples generated by sampling with replacement to produce a diverse set of models which can then be combined to promote stability \cite{Bagging1996}. Though a given model may be over-fit to its sample, the combination should less reflect those unstable aspects particular to an individual sample \cite{buchlmann2002analyzing}. Perhaps its most important use in pattern recognition has been in Random Forest classification and regression \cite{RF}, where it is key to allowing the use of flexible decision trees without over-fitting. While bagging is still most commonly performed by drawing $N$ items with replacement, the way in which the samples are drawn is not necessarily fixed. The number of unique items present in a bootstrap sample has been identified as an important predictor of algorithm performance, and the sampling method can be purposefully modified to control its range and variability \cite{buchlmann2002analyzing,Reduced,Outlier}.  

An added bonus of bagging is the possibility of producing a performance estimate for the ensemble without doing additional cross validation \cite{1996OOB,RF}. Predictions are made for items in the training sample using only those models which did not use them in training. In this way, it is possible to get a performance estimate for an ensemble without the bias of over-fitting. A pessimistic bias has been observed in high dimensional classification problems that was ameliorated using different sampling techniques \cite{mitchell2011bias}, suggesting the importance of unique items in this phenomenon.

\subsection{Bootstrap performance estimators}

There is no one dominant scheme for the statistical comparison of learning algorithms on small to moderately sized datasets, and researchers must choose between different forms of hold out tests and cross validation. One option available to them is the .632+ bootstrap of \cite{632p}. This, and the earlier .632 scheme \cite{632}, extend the bootstrap to the validation of learned prediction rules. Both of these are functions of a quantity called the leave-one-out bootstrap error, which is the average error of the prediction models constructed on many bootstrap samples. To avoid the optimistic bias of testing on the training set, the models are tested only on those points not in their respective bootstrap samples. 

As in bagging, the prediction rules created can have an effectively reduced training sample size in line with the lower number of unique original items in their samples. The expectation of the leave-one-out error then becomes a weighted average of the expected algorithm performance at the different sizes given by the distribution of that quantity. Unlike conventional cross validation schemes where the reduced sample size is obvious, users of this technique may not have the effective reduction in mind. While both the .632 and .632+ methods attempt to correct for the pessimistic bias that may result, the true corrective mapping of this value will be problem specific and cannot be known ahead of time. This is almost certainly the cause of the bias observed by \cite{kim2009}. Knowledge of the distribution of the effective sample sizes in these schemes should be useful to those interpreting the results of these schemes, particularly when there is concern about bias.

\section{The distribution of unique items}
\label{present_dist}

In this section, we first provide a closed form of the distribution of the number of unique items and describe its connection to the family of occupancy distributions. With this information, we are then able to derive the integer moments and pormal limit. Finally, we describe the construction of a heuristic to decide when normal approximation is appropriate. Some illustrations of the distribution are given in Fig. \ref{onedim}.

\begin{defin}
    Where $A$ items are taken with replacement from a sample of size $N$, the probability of obtaining $k$ unique original items is
\begin{equation}
    \label{initialdefine}
    \probP ( k ) = \frac{\falling{N}{k}}{N^A} \stirling{A}{k},
\end{equation}
where $N,k,A \in \mathbb{N}_{0}$, $\falling{N}{k}$ is the falling factorial, and $\stirling{A}{k}$ is a Stirling number of the second kind. Equivalently, this is the probability of obtaining $k$ unique outcomes after $A$ categorical trials in a multinomial problem with $N$ equally likely outcomes.
\end{defin}

\begin{derivation}
Because each of the $N$ items is equally likely to be drawn at each of the $A$ sampling events, there are $N^A$ equally likely bootstrap samples that may be drawn if order is accounted for. The number of ways to make an ordered selection of $k$ unique items from the original $N$ to be present in the bootstrap sample is $\falling{N}{k}$, the falling factorial. This is defined as follows: 
\begin{equation}
    \label{falleq}
    \falling{N}{k} =
    \begin{cases}
        \frac{N!}{(N-k)!}, & \text{if}\,\, k \leq N \\
        0, & \text{otherwise.}
    \end{cases}
\end{equation}
Having already considered the ordering of the $k$ unique items to be included means that we may consider them unlabelled in the final step; here, we consider the number of ways to take $A$ distinguishable draws from $k$ unlabelled unique items such that at least one draw is taken from each. In more generic terms, this is the number of ways to place $A$ labelled items into $k$ subsets such that none of them are empty. This quantity is written as $\stirling{A}{k}$, and is called a Stirling number of the second kind \cite[Chapter~6]{Concrete}. Combining these three results gives us Eq. \eqref{initialdefine}.
\end{derivation}

\subsection{Basic properties and view as a sum of indicators}
\label{indicators}
The number of unique items present may be viewed as the sum of the set of $N$ indicator functions ${d_1,d_2,...,d_N}$. The indicator $d_i$ corresponds to original item $i$, taking the value of one if this is present and zero if not. A simple consideration of the sampling process allows one to find the mean and covariance of the indicator functions. As all points are equally likely to be selected, the chance of an item not being selected at a given sampling trial is $1/N$. In order for $d_i$ to be zero, its corresponding point must not be selected at all $A$ events. This gives us $\probP(d_i = 0) = (1 - 1/N)^A$. As the indicators are binary variables, this is sufficient to determine their mean and variance:

\begin{equation}
    \label{indicator_mean}
    \begin{split}
        \text{E}[d_i] &= \probP(d_i = 1) = 1 - \probP(d_i = 0) \\
               &= 1 - \Big( 1 - \frac{1}{N} \Big)^A, \\
      \text{Var}[d_i] &= \probP(d_i = 0) \cdot \probP(d_i = 1) \\
               &= \Big( 1 - \frac{1}{N} \Big)^A - \Big( 1 - \frac{1}{N} \Big)^{2A}.
    \end{split}
\end{equation}

A similar consideration gives us the probability of two items, $i$ and $j$ (where $i \neq j$) both being excluded from a bootstrap sample. This is given $\probP( d_i = 0, d_j = 0 ) = (1 - 2/N)^A$, and is enough to give us the covariance of two indicator functions,

\begin{equation}
    \label{indicator_cov}
    \begin{split}
        \text{Cov}[d_i,d_j] &= \probP( d_i = 1, d_j = 1 ) - \Big(\probP(d_i = 1)\Big)^2 \\
                     &= \Big(1-\frac{2}{N}\Big)^A - \Big(1-\frac{1}{N}\Big)^{2A}.
    \end{split}
\end{equation}

Now that we know the mean, variance and covariance of the indicators, we have enough information to determine the same quantities for any linear combination of them, including their sum, $k$. Of particular interest to us is the case where the number of samples drawn is proportional to the number of original items. For this reason, along with the general formulae, we also provide limits for the mean and variance as $N$ and $A$ jointly approach infinity:

\begin{equation}
    \begin{split}
        \text{E}[k] &= N(1 - \Big(1 - \frac{1}{N}\Big)^A) = \frac{ N^A - (N-1)^A }{ N^{A-1} } \\
             &\to N(1 - e^{-\alpha}), \\
        \text{Var}[k] &= N(N-1)\Big(1-\frac{2}{N}\Big)^A + N\Big(1 - \frac{1}{N}\Big)^A - N^2\Big(1-\frac{1}{N}\Big)^{2A} \\
               &\to N(e^{-\alpha} - (1+\alpha)e^{-2\alpha}),
    \end{split}
\end{equation}
where the limits refer to the case $A = \alpha N$ and $N \to \infty$.

\begin{figure}[!ht]
\centering
\includegraphics[width=0.55\textwidth]{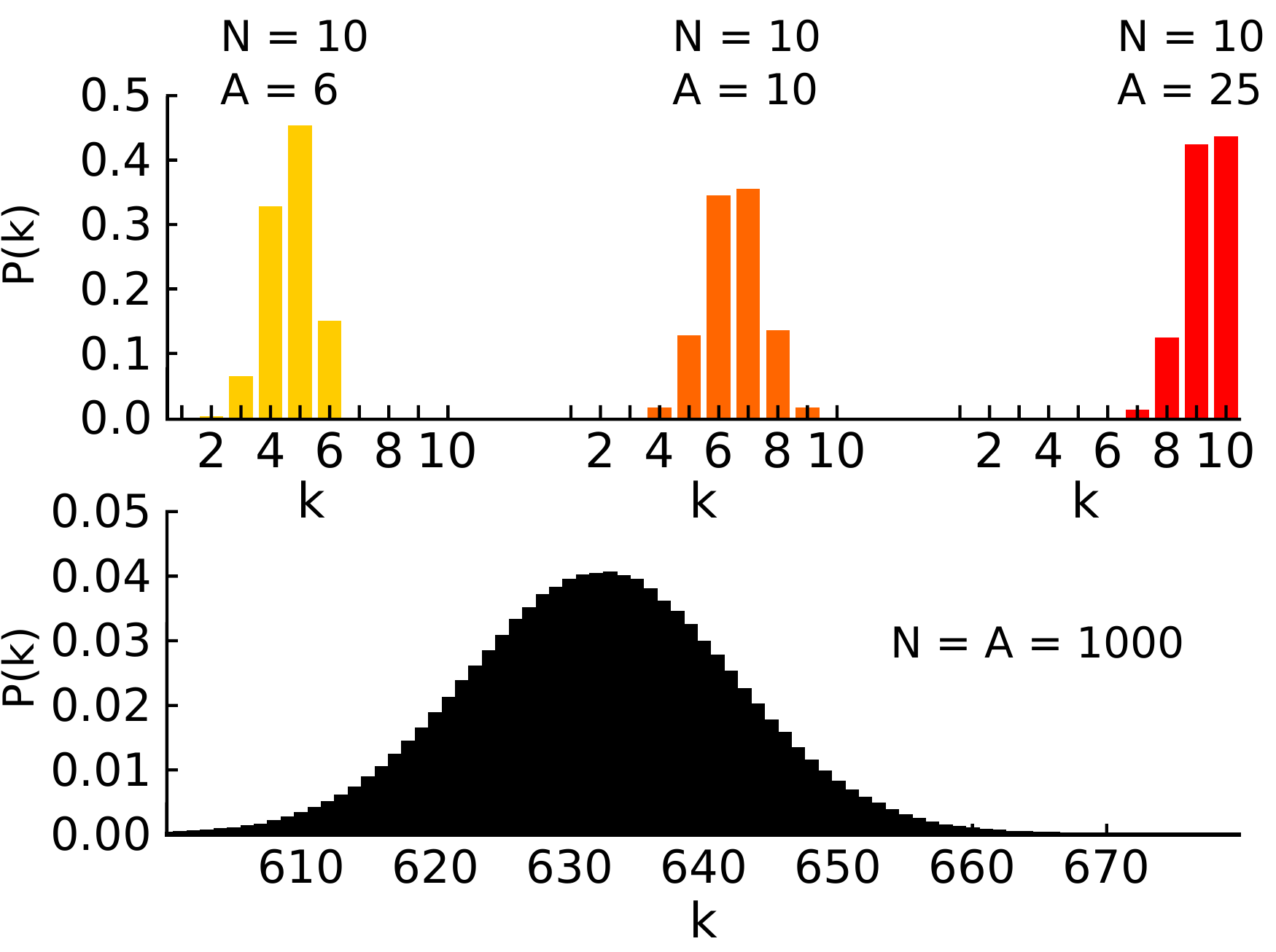}
\caption{Example of the distribution, and its convergence to normal.\label{onedim}}
\end{figure}

\subsection{ Relation to study of urn problems }

While we came to the problem of item inclusion through machine learning, the study of this problem far predates our field of research \cite{Urns,Weiss}. If, instead of the number of items included in the bootstrap sample, we were to consider the number of items \textit{excluded}, we would be studying one of the family of occupancy distributions \cite{Urns}. These arise in the study of urn problems when $A$ balls are independently and randomly placed into $N$ urns in such a way that each ball is equally likely to be placed into any one of the urns. The $i^{th}$ occupancy distribution then details the probabilities of obtaining different values of $m_i$, the number of urns containing exactly $i$ balls. In our problem, the $A$ drawing events are the balls of the urn model, and the $N$ original items that may be selected are the urns with which the balls/draws can be associated. The number of empty urns is $m_0$, and it has the same distribution as the number of excluded items in our sampling problem. As $k$ is just $N-m_0$, it is easy to drawn upon useful results.

\subsection{Asymptotic normality}

While there exist a great many limit theorems for occupancy distributions \cite[Chapter~6]{Urns}, in our case we need only the earliest. \cite{Weiss} proved that the limit distribution of $m_0$ was normal in the case that $A = \alpha N$, and $N \to \infty$. As $k$ is distributed as $N-m_0$, it too must be normally distributed. 

\subsection{Integer moments} 
The $t^{th}$ raw moment of the distribution of $k$ is
\begin{equation}
\label{moments}
\begin{split}
    \mu'_t &= \sum_{k = 0}^{N}  k^t \frac{\falling{N}{k}}{N^A} \stirling{A}{k} \\
          &= \sum_{\substack{u,v,w \in \mathbb{N}_{0}\\ u + v + w = t}} \nchoosek{ t }{ u } (-1)^v \stirling{v+w}{v} \falling{N}{v} (N-v)^{u} \Big( 1-\frac{v}{N} \Big)^{A}.
\end{split}
\end{equation}
A proof for this is provided in appendix \ref{rawproof}. \cite{Weiss} provides a formula for the central moments of the $0^{th}$ occupancy distribution, which for us corresponds to the number of excluded items. By simplifying this and adapting the sign to the distribution of the number of items \textit{included}, we can give the $t^{th}$ central moment as
\begin{equation}
    \begin{split}
          \mu_t &= \sum_{k = 0}^{N}  (k- \mu'_1) ^t \frac{\falling{N}{k}}{N^A} \stirling{A}{k} \\
                &= \sum_{\substack{u,v,w \in \mathbb{N}_{0}\\ u + v + w = t}} \nchoosek{ t }{ u } (-1)^{v+w} \stirling{v+w}{v} N^{u} \falling{N}{v} \Big( 1-\frac{v}{N} \Big)^{A} \Big( 1-\frac{1}{N} \Big)^{uA}.
    \end{split}
\end{equation}

In both of these formulae the sum over $u,v$ and $w$ comprises $(t+1)(t+2)/2$ unique locations. Reassuringly, they provide the identity $\mu_0 = \mu'_0 = 1$, as well as results for the mean and variance that match up to those derived by the considerations of \ref{indicators}. 

\subsection{Building a heuristic for normal approximation}

\label{heuristic}

The Stirling numbers quickly overflow conventional precision, making the form of the distribution in  Eq. \eqref{initialdefine} hard to work with at reasonable values of $N$ and $A$. While it is known that the distribution does converge to normal, this is not practically useful unless one knows when a normal approximation is appropriate. To allow others to use a normal approximation with confidence, we decided to build a heuristic based on existing rules of thumb for normal approximation of the binomial distribution. 

Where the number of Bernoulli trials is $n_b$ and the probability of success is $p$ at any one, a normal approximation works best when the $n_b$ is high and $p$ is not too close to zero or one. While there is a degree of arbitrariness in deciding when two distributions are ``close enough'', there have long existed conventions for doing so. Two of the most common such rules from applied statistics \cite{box1978statistics,Thumb} are 
 \begin{equation}
     \label{binrule}
     \begin{split}
         &n_b p > 5 \text{ and } n_b (1-p) > 5,  \text{ and}\\
         &  \frac{| 1 - 2p | }{p(1-p)} < 0.3 \sqrt{n_b} \,\, \text{ and } \,\, n_b > 5.
     \end{split}
  \end{equation}
In order to build our own rule for approximation of $\probP(k)$, we decided first to measure the convergence of the true distribution and its normal approximation to measure the quality of the approximation of the binomial case under the rules of \eqref{binrule}. We could then experiment with building rules to permit approximation of the distribution of $k$ such that convergence was always better than the worst case seen for the binomial approximation. To measure the quality of approximation, we chose a metric closely linked to the concept of convergence in distribution: the maximum absolute difference in cumulative distribution (MADCD) between the true distribution and its continuity-corrected normal approximation on the set of all possible outcome values. The continuity correction was performed using a shift of $0.5$. Computation of the exact value of the true distribution was performed using the symbolic mathematics toolkit of the MATLAB software package\footnote{http://uk.mathworks.com/products/matlab/}.

To find the worst permissible binomial approximation, we measured the MADCD between the binomial distribution and its continuity-corrected normal approximation across the grid of parameters specified by $n_b \in \{1,2,...,400\}$ and $p \in \{1E-5,2E-5,...,0.5\}$. We combined the two rules of Eq. \eqref{binrule} to produce a third, stricter rule, and found a maximum MADCD of 0.0205 within the resulting acceptance region. As this value occurred far from the boundary of the allowed set, we think it unlikely that expanding the grid would produce higher values. 

We then mapped the MADCD for $\probP(k)$ and its continuity-corrected normal approximation across the grid of parameters $A = {1,2,...,400}$ and $N = {1,2,...,400}$, so that we could begin constructing a rule. The results can be seen in Fig. \ref{MADCD_GRID}. After inspecting this map, we decided to construct rules for the minimum permissible $A$ at a given $N$ (and vice versa), of the form  $\ln(Y_{min}) \geq a \ln(X) + b$. While there are areas near $A << N$ and $N << A$ where the MADCD value is low, this did not represent convergence to normal, rather, most of the probability mass has collapsed to a single value of $k$. For this reason, we chose to focus on the central valley of good approximation near the axis $A=N$.

\begin{figure}[!ht]
\centering
\includegraphics[width=0.55\textwidth]{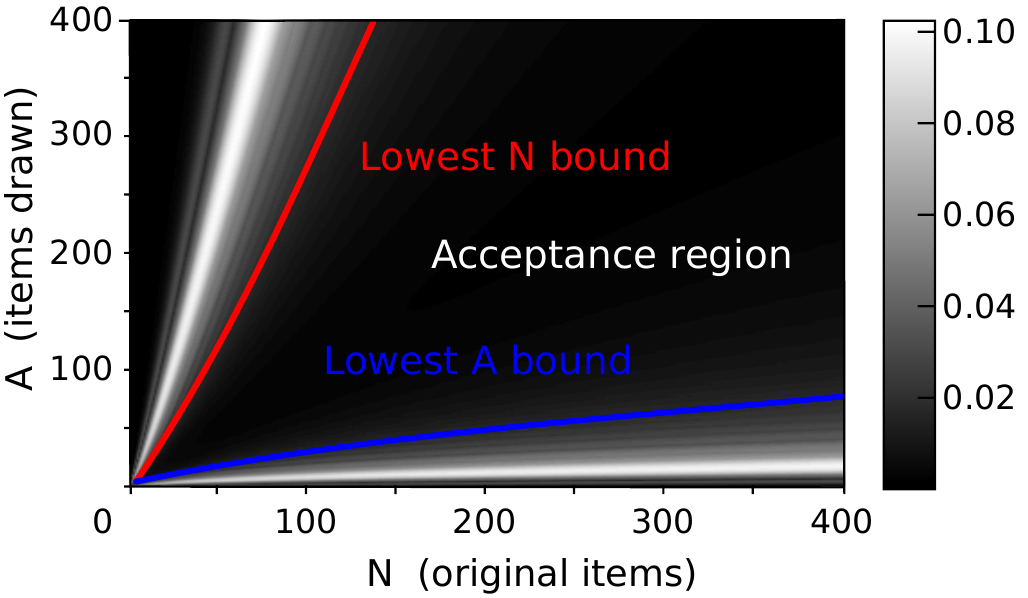}
\caption{MADCD measured across parameter space, with boundaries and acceptance region for normal approximation illustrated. \label{MADCD_GRID}}
\end{figure}

We were mainly concerned with producing a reliable rule that would not fail at parameter setting outside the mapped grid; while we experimented with different fitting methods, we chose to select the rule parameters by hand so as to have greater control over features likely to make it more conservative. This also made it easier to restrict the rule parameters to those fully specified by two digits of decimal precision. We produced a rule such that 
\begin{itemize}
    \item all MADCD values in the acceptance region were below the specified level of 0.0205,
    \item the boundaries of the acceptance region did not cross the ridges of the distribution as $N$ and $A$ increase, unless it was to move towards the central valley, and
    \item the values on the boundary had to appear to be decreasing with increasing $N$ and $A$.
\end{itemize}

We excluded the smallest values of $A$ and $N$ in order to better fit of the general pattern. The resulting rule is that a normal approximation is permissible when
\begin{equation}
    1.4 N^{0.67} \leq A \leq 1.13 N^{1.19}, \text{ } N > 5, \text{ and } \, A > 5.
\end{equation}
The acceptance region specified and the values of MADCD at its boundaries can be seen in Fig. \ref{MADCD_GRID} and \ref{MADCD_GRID_BOUND} respectively. 

In addition to the MADCD, we also computed the Jensen-Shannon divergence between the true probability distribution and the normal approximation. As this is function of a probability distribution itself, rather than a cumulative distribution, it was necessary to discretise the normal approximation. This was done by defining the probability mass at a location as the difference in the continuity-corrected cumulative distribution between that location and the next. We found a maximum value of 0.0444 within the allowed region for the binomial distribution, and 0.0631 for the allowed region for the distribution of $k$. All the highest divergence values for $\probP(k)$ and its approximation occurred on the boundaries near the very lowest allowed values of $A$ and $N$.

\begin{figure}[!ht]
\centering
\includegraphics[width=0.55\textwidth]{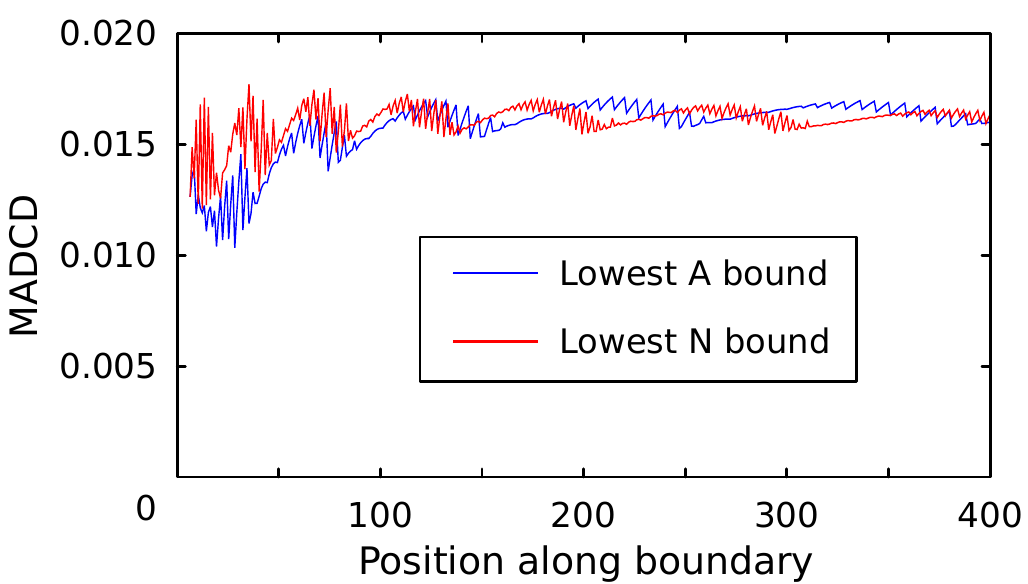}
\caption{MADCD value along boundaries of acceptance region. \label{MADCD_GRID_BOUND}}
\end{figure}

\section{Unique items from multiple categories}
\label{multi}

In some machine learning tasks, notably in supervised classification, items in the sample belong to one of several categories. It is then of interest to those studying bootstrap techniques to know the distribution of the balance of the unique items drawn from each category. We consider the case where there are $C$ categories present, the $s^{th}$ of which contains $N_s$ items in the original sample of $N$. When $A$ samples are drawn without replacement, the distribution can be found by marginalising over the number of draws taken from each category, the various $a_s$. The probability of obtaining the vector $\mathbf{k} = \{k_1,k_2,...,k_C\}$ detailing the number of items $k_s$ from each of category is

\begin{equation}
    \probP ( \mathbf{k} ) = A! \sum_{ {\substack{ a_1,...,a_c \in \mathbb{N}_0 \\ a_1 + ... + a_C = A }}} \prod_{s=1}^{C} \frac{\falling{N_s}{k_s} }{ N_{s}^{a_s} a_{s}!} \stirling{a_s}{k_s}.
\end{equation}

An example of this distribution in the case of two categories is illustrated in Fig. \ref{twodim}.

\begin{figure}[!ht]
\centering
    \includegraphics[width=0.55\textwidth]{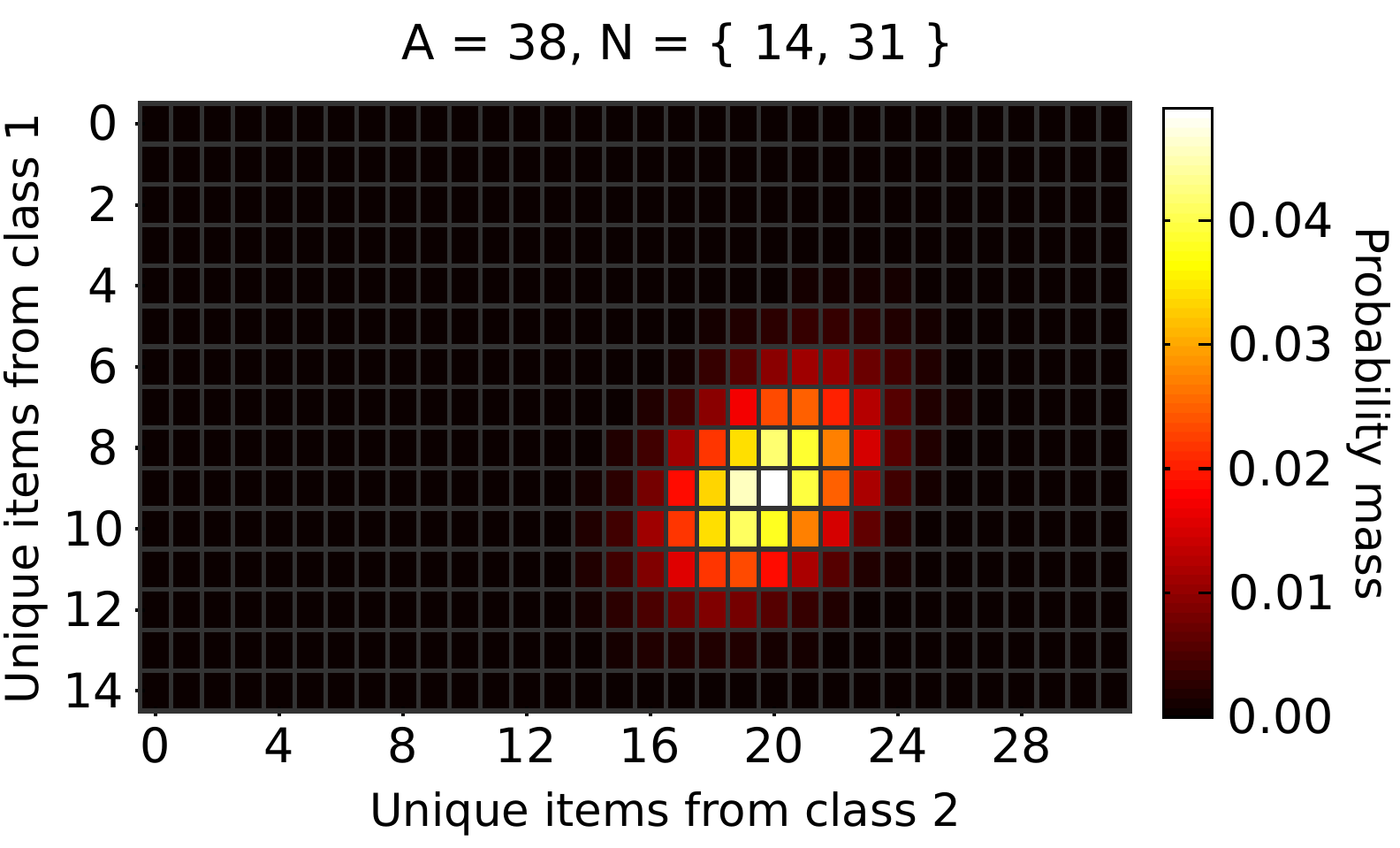}
\caption{Example of the joint distribution of the number of unique items from two categories.\label{twodim}}
\end{figure}

\subsection{ Mean and variance }
From consideration of the mean and covariance of the individual items' participations as in section \ref{indicators}, it is straightforward to derive the mean, variance and covariance; these are the following:
\begin{equation}
    \begin{split}
        \text{E}[k_i] &= N_i (1 - \Big(1 - \frac{1}{N}\Big)^A), \\
        \text{Var}[k_i] &= N_i(N_i-1)\Big(1-\frac{2}{N}\Big)^A + N_i\Big(1 - \frac{1}{N}\Big)^A - N_i^2\Big(1-\frac{1}{N}\Big)^{2A}, \\
        \text{Cov}[k_i,k_j] &= N_iN_j\Big[\Big(1-\frac{2}{N}\Big)^A - \Big(1-\frac{1}{N}\Big)^{2A} \Big], \text{ for $i \neq j$.}
    \end{split}
\end{equation}

\subsection{Asymptotic normality and approximation}

Not only is the total number of unique items from the original sample asymptotically normally distributed under appropriate conditions, but so is the number from the subset belonging to each category. To show this, we must turn to limit theorems for dependent variables rather than the study of urn problems. The covariance between the presence indicators of section \ref{indicators} is always negative (see Eq. \eqref{indicator_cov}). As these are binary variables, this means that the joint distribution of any two indicators $d_i$ and $d_j$ ($i \neq j$) must meet the condition
\begin{equation} 
\probP( d_i < D_1, d_j < D_2 ) \, \leq \, \probP( d_i < D_1 ) \cdot \probP( d_j < D_2) ,
\end{equation}
for any values $D1$ and $D2$. Any collection of the indicators can then be termed a pairwise negative quadrant dependent (PNQD) sequence \cite{LiPNQD}. Weighted sums of PNQD sequences are asymptotically normally distributed as the number of items in the sequence goes to infinity \cite{LiPNQD}. The number of unique items present from the subset of a category is a weighted sum of the indicators, and so its asymptotic distribution is normal as the number of items present and items drawn jointly approach infinity. The number of items in a category must approach infinity too, as the number of possible values that the sum can take must also go to infinity for its discrete distribution to converge to a continuous one. This is guaranteed in the case where each category contains a fixed fraction of the total number of items.

To show the limit of the joint distribution of the numbers of unique items from the categories is multivariate normal, we can consider the set of weighted sums that give equal weight to all items within a category. If the weight given to items in category $s$ is $\gamma_s$ and we add the constraint $\sum_s \gamma_s = 1$, then the distributions of this set still include all possible projections of the multivariate distribution of $\mathbf{k}$. As all projections of this distribution are normal, then it itself must be multivariate normal \cite{intermediate}.

\subsubsection{ Approximation heuristic for multivariate case }

We now consider when it is appropriate to approximate the distribution of the number of unique items from a subset of the data with a normal distribution. We consider the case where the original items has $N$ items of which $N_s$ belong to a particular subset of fixed size, and where the number of items drawn is $A = \alpha N$. In the case where $N_s = N$, we have the heuristic developed in \ref{heuristic}. As $N \to \infty$ while $N_s$ remains fixed, the correlation between the indicator functions of the items of category $s$ goes to zero and the distribution of $k_s$ approaches a binomial distribution with $N_s$ trials and a probability of success $1 - (1 - 1/N_s)^A$. We argue that if both the distribution with $N = N_s$ and that with $N \to \infty$ are well approximated by the binomial, it is highly likely that the intermediate stages are too.

This suggests the following rule for normal approximation of $k_s$: if the number of samples drawn is proportional to the number of items ($A = \alpha N$) and the expected number drawn from the category $s$ is $\bar{a_s} = \alpha N_s$, then if both the rules for approximation of the binomial limit (with $p = 1 - e^{-\alpha}$ and $n_b = N_s$) are met, and the rule for the single category case with $N$ and $A$ substituted for $N_s$ and $\bar{a_s}$ respectively, then a normal approximation should be appropriate for $k_s$ for any value of $N$. 

The resulting rule requires four inequalities to be met, which makes it a little unintuitive and tedious to apply. As the inequalities due to the consideration of the binomial limit only have a small effect on the acceptance region near the limit of lowest $N_s$ / highest $\bar{a}_s$, and eventually become redundant as $N_s$ and $\bar{a}_s$ increase, it is possible to account for them with a simple offset of the rule for the single category case. A small strip of $\bar{a}_s$--$N_s$ parameter space is then sacrificed for the sake of simplicity. We find a normal approximation to be permissible when
\begin{equation}
1.4 N_s^{0.67} \leq \bar{a}_s \leq 1.13 (N_s-8)^{1.19}. 
\end{equation}
This has no solutions with $\bar{a}_s < 9$ or $N_s < 14$.

\section{Conclusions}

We have aimed to make a clear and concise answer to the titular question readily available to machine learning researchers. We have summarised the key properties of this distribution, and provided practical information about when a normal approximation is appropriate in the form of a heuristic, allowing others to justify its use. With particular consideration to classification problems, we have considered the generalisation of this distribution to the scenario where items come from multiple categories. In this case, we have provided a theoretical guarantee of asymptotic normality and a considered heuristic rule for approximation based its limit distributions. We hope our results rovide a useful resource to researchers interested in understanding or modifying the number of unique items to appear under random sampling with that replacement.

\section*{Acknowledgements} \begin{small} This work is funded by UCL (code ELCX), a CASE studentship with the EPSRC and GE healthcare, EPSRC grants (EP/H046410/1, EP/\- H046410/1, EP/J020990/1, EP/K005278), the MRC (MR/J01107X/1), the EU-FP7 project VPH-DARE@IT (FP7-ICT-2011-9-601055), the NIHR Biomedical Research Unit (Dementia) at UCL and the National Institute for Health Research University College London Hospitals Biomedical Research Centre (NIHR BRC UCLH/UCL High Impact Initiative). \end{small}

\appendix
\section{Derivation of raw integer moments\label{rawproof}}
Before proceeding, we shall need two intermediate results. 
\begin{lemma} The multiplication of a Stirling number $\stirling{A}{k}$ by its second argument raised to an integer power $t$ may be expanded as
\begin{equation}
    \label{triangle_lemma}
k^t \stirling{A}{k} =  \sum_{\substack{ u,v,w \in \mathbb{N}_{0} \\  u + v + w = t}} \nchoosek{t}{u} (-1)^v \stirling{v+w}{v} \stirling{A+u}{k-v}.
\end{equation}
\end{lemma}

\begin{proof} We begin by modifying the recurrence relation of the Stirling numbers \cite{Concrete} by an offset in $k$ to provide the following identity:

\begin{equation}
    \label{Kstirling}
    k \stirling{A}{k - m} = \stirling{A+1}{k-m} - \stirling{A}{k-m-1} + m \stirling{A}{k-m}
\end{equation}
Looking at Eq. \eqref{Kstirling}, we can see multiplication by $k$ as triplicating a Stirling number and coefficient entity by applying three linear operators to it:
\begin{itemize}
    \item[$\hat{u}$] Add one to the upper argument.
    \item[$\hat{v}$] Subtract one from the lower argument and multiply by minus one.
    \item[$\hat{w}$] Multiplication by offset between $k$ and the second argument of the Stirling number (i.e., the initial $m$ plus the number of $\hat{v}$ operations that have occurred to the number thus far).
\end{itemize}
By multiplying by $k^t$, we sequentially apply all these operations $t$ times. In a way similar to Pascal's triangle, we can see this triplication as carrying a ``charge'' of Stirling numbers down to the next level of a pyramid (see \textit{a)} in Fig. \ref{pyramid}. Locations in the pyramid can be described by the number of times each operator has been applied to all charge that reached them. We shall call these coordinates $u$, $v$ and $w$. Any given selection of $t$ populates the pyramid down to the level described by $u+v+w=t$. At any point in the pyramid, the arguments of the Stirling numbers are determined only by the coordinates.

If these operators commuted, we could consider a flow of charge exactly the same as in Pascal's triangle, and apply the effects of the operators on the coefficients of the Stirling numbers separately at the end. As $\hat{v}$ and $\hat{w}$ do not commute with one another, we must consider their interaction and how they change the charge passing through them. $\hat{u}$ commutes with the other operators and does not change the coefficient of the charge, so we can ``factorise'' the pyramid. That is, we need only consider the number of ways to select $t - u$ non-$\hat{u}$ operations, and then to sum over all possible orders of choosing $v$ and $u$ operations of $\hat{v}$ and $\hat{w}$ respectively. 

The amount of charge at a location in the pyramid will be $\nchoosek{t}{u} f(v,w)$, where $f(v,w)$ is the coefficient from the charge triangle of the operators $\hat{v}$ and $\hat{w}$ alone. To determine what this is, we use the fact that the $m$ in Eq. \eqref{Kstirling} is initially zero. See \textit{b} in Fig. \ref{pyramid} for illustration. This gives us the initial conditions $f(0,0) = 1$ and $f(v,0) = 0$ for $v > 0$. Together with the recurrence $f(v,w) = kf(v,w-1) - f(v-1,w)$, these are enough to uniquely specify $f(v,w) = (-1)^v \stirling{v+w}{v}$. This then gives us Eq. \eqref{triangle_lemma}.

\begin{figure}[!t]
\centering
    \includegraphics[width=0.55\textwidth]{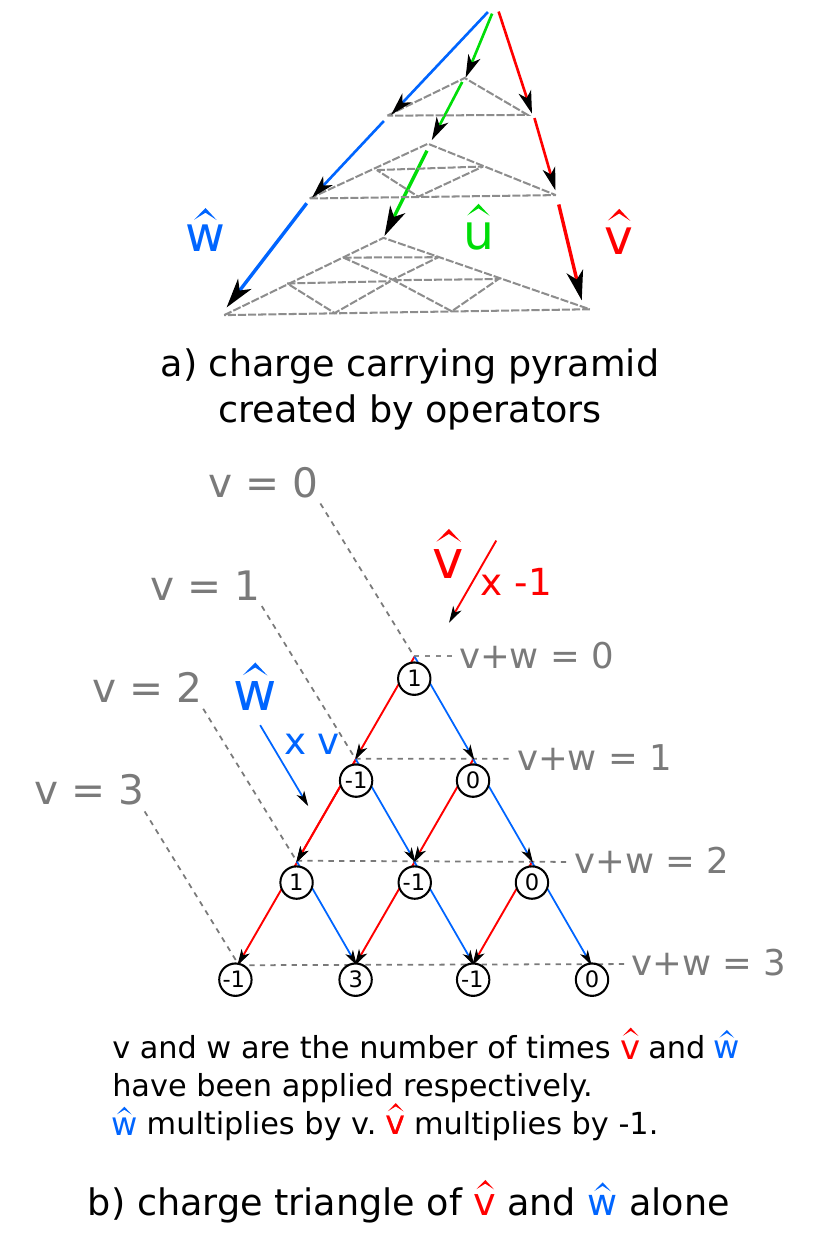}
\caption{Illustration of the flow of Stirling numbers via the operators. \label{pyramid} }
\end{figure}

\end{proof}

\begin{lemma} The standard relation between summations of the Stirling numbers and falling factorials powers may be modified by an integer offset $w$ in the second argument to produce the following:
\begin{equation}
    \label{falling_lemma}
    \sum_{k = 0}^{K} \falling{N}{k} \stirling{A}{k-w} = \falling{N}{w} (N-w)^A,  \text{ providing } K \geq \text{min}(N,A).
\end{equation}
\end{lemma}

\begin{proof}
We begin with the standard relation \cite[p.~264]{Concrete}
\begin{equation}
\label{standardsum}
\sum_{k = 0}^{A} \falling{N}{k} \stirling{A}{k} = N^A.
\end{equation}
From Eq. \eqref{falleq}, it is clear that $\falling{N}{k} = 0$ for values of $k$ higher than $N$. Consistent with their recurrence relation and combinatorial meaning, the Stirling numbers $\stirling{A}{k}$ can be defined as zero for positive values of $k$ outside $\{0,1,2,...,A\}$ \cite{Concrete}. As one of these two factors is zero whenever $k$ exceeds either $A$ or $N$, we can truncate the sum in \ref{standardsum} to the smaller of those without effect. We can therefore write
\begin{equation}
    \label{newsum}
    \sum_{k = 0}^{K} \falling{N}{k} \stirling{A}{k} = N^A, \text{ providing } K \geq \text{min}(N,A).
\end{equation}
If we combine this with the knowledge that
\begin{equation}
    \label{fallsplit}
    \falling{N}{k + m} = \falling{N}{m} \falling{N-m}{k},
\end{equation}
we can then produce Eq. \eqref{falling_lemma}.

\end{proof}

\begin{fproof}

If we take the definition of the moment as in the first line of Eq. \eqref{moments}, apply Eq. \eqref{triangle_lemma}, reverse the order of summation, and then apply Eq. \eqref{falling_lemma}, we are then able to derive the final form for the raw integer moments seen in the second line (of Eq. \ref{moments}).

\end{fproof}

\end{document}